\documentclass{article}

\PassOptionsToPackage{numbers, compress}{natbib}

\usepackage[final]{neurips_2025}

\usepackage[utf8]{inputenc} 
\usepackage[T1]{fontenc}    
\usepackage{hyperref}       
\usepackage{url}            
\usepackage{booktabs}       
\usepackage{amsfonts}       
\usepackage{nicefrac}       
\usepackage{microtype}      
\usepackage{xcolor}         

\usepackage{amsmath}
\usepackage{graphicx}
\usepackage{subfig}
\usepackage{multirow}
\usepackage{amsthm}
\newtheorem{theorem}{Theorem}
\usepackage{graphicx}
\usepackage{subcaption}
\usepackage{caption}
\usepackage{float}

\usepackage{amsthm}
\newtheorem{proposition}{Proposition}

\usepackage{color, colortbl, xcolor}
\newcommand\siyu[1]{\textcolor{black}{#1}}

\title{Memory-Efficient Training with In-Place FFT Implementation}

\author{
  Xinyu Ding\textsuperscript{\rm 1},
  Bangtian Liu,
  Siyu Liao\textsuperscript{\rm 1}\thanks{Siyu Liao is the corresponding author.},
  Zhongfeng Wang\textsuperscript{\rm 1} \\
  \textsuperscript{\rm 1}School of Integrated Circuits, Sun Yat-sen University \\
  \texttt{\{dingxy33, liaosy36, wangzf83\}@mail.sysu.edu.cn},
  \texttt{liubangtian@gmail.com},
}

\begin{document}

\maketitle

\begin{abstract}
Fast Fourier Transforms (FFT) are widely used to reduce memory and computational costs in deep learning. However, existing implementations, including standard FFT and real FFT (rFFT), cannot achieve true in-place computation. In particular, rFFT maps an input of size $n$ to a complex output of size $n/2+1$, causing dimensional mismatch and requiring additional memory allocation.
We propose the first real-domain, fully in-place FFT framework (rdFFT) that preserves input-output memory space consistency. 
By leveraging butterfly operation symmetry and conjugate properties in the frequency domain, we design an implicit complex encoding scheme that eliminates intermediate cache usage entirely.
Experiments on multiple natural language understanding tasks demonstrate the method  effectiveness in reducing training memory cost, offering a promising direction for frequency-domain lightweight adaptation.
\end{abstract}

\section{Introduction}

Large-scale neural models have achieved remarkable success across a wide range of applications, such as natural language processing \citep{GPT-3}\citep{T5} and computer vision \citep{ViT}. 
\siyu{As model sizes increase, memory consumption has emerged as a significant challenge.
Notably, model training memory cost is larger than the deployment cost, primarily due to the backpropagation process \cite{lecun1988theoretical}. 
Thus, reducing memory usage has become a critical research study, especially for the training stage. 
}

Numerous methods have been proposed to reduce memory usage during model training and deployment.
\siyu{Commonly used approaches such as model distillation \cite{hinton2015distilling}, quantization \cite{courbariaux2015low}, and pruning \cite{DBLP:journals/corr/HanMD15} mainly aim to reduce memory consumption by decreasing the number or precision of model parameters. 
In contrast, this work takes a novel approach at the arithmetic operator level, implementing in-place Fast Fourier Transform (FFT) operations for model training through the use of circulant-structured parameter matrices \cite{cheng2015exploration,ding2017circnn,circulant}.
}
The FFT operator has been widely employed in various neural network architectures due to its efficiency in capturing global patterns and enabling structured transformations. 
For instance, Fourier-based fine-tuning methods such as FourierFT \citep{fourierft} and  Block Circulant Adapter (BCA) \citep{circulant} rely heavily on FFT to perform efficient parameter transformations in the frequency domain. 

\siyu{It should be noted that FFT-based operators typically involve both FFT and IFFT computations, which generate intermediate tensors that consume memory and use different data type since the result of FFT is complex number.
}
Numerous studies have explored memory optimization strategies for FFT operators, particularly in high-performance programming libraries such as FFTW \citep{fftw} and cuFFT \cite{cufft_layout}. 
\siyu{These libraries support in-place computation, where input and output share memory space.
However, they suffer from the inability to maintain the original (input) memory space and lack of support for bfloat16 data type that is widely used for modern neural models. 
}



\siyu{To overcome the limitations of existing FFT libraries, we introduce rdFFT—a real-domain, fully in-place Fourier transform that produces the same output as rFFT, but operates entirely within the original $n$ real-valued input memory space.
We notice that the first and middle point of FFT results have zeros in their imaginary part.
Thus, they can be squeezed together so the final output only requires $n$ real-valued input memory space. 
}
Besides, our method exploits the symmetry of butterfly operations and the conjugate structure of real-valued spectra to implicitly encode complex information within real-valued tensors. 
This design enables in-place computation without the need for auxiliary buffers or dimension mismatches, facilitating seamless integration into modern deep learning workflows. 
Crucially, our in-place design supports consistent forward and backward passes entirely within the real domain.
In summary, our main contributions are as follows:
\begin{itemize}
    \item \siyu{We propose rdFFT, a novel real-valued, fully in-place Fourier transform operator that eliminates memory space mismatches, offering improved practicality and usability for neural network applications compared to existing libraries.}
    \item \siyu{We introduce a new memory layout design, develop a novel butterfly execution scheme for IFFT computation, and provide support for the bfloat16 data type, which is widely used in modern neural networks.}
    \item \siyu{We integrate our rdFFT into neural network models via circulant-structured parameter matrix and validate on real-world models, achieving zero-memory allocation for intermediate tensor computations.}
\end{itemize}

The implementation of the proposed operators is being discussed for upstreaming into PyTorch. The corresponding discussion is available at:
\url{https://github.com/pytorch/pytorch/issues/171022}.

\section{Related Works}

Current automatic differentiation frameworks often discourage in-place operations due to their challenges in gradient computation during training.  
The key benefit of in-place computation lies in its ability to save memory by eliminating the need for storing intermediate tensors. 
For instance, combining batch normalization and activation into a single in-place operation has been shown to reduce memory usage by up to half \citep{bulo2018place}. 
In neural network based hardware placement tasks \cite{liu2022xplace}, in-place operations are found beneficial due to their inherent memory efficiency. 
Similarly, in-place operations are integrated into convolutional neural networks for anomaly detection \cite{sun2023tinyad}, achieving notable memory savings.
The in-place operation can also be extended to broader contexts. 
For example, in in-place distillation \cite{yu2019universally}, a student model is distilled directly from the  teacher model without additional memory allocation.

The FFT operator can be found across various neural network models, particularly in tasks involving the Fourier domain, which is common in computer vision. 
Some convolutional neural networks run entirely in the Fourier domain \cite{pratt2017fcnn}. 
In beampattern synthesis \citep{zhao2023efficient}, an IFFT operator is applied to the hidden representations generated by the neural network. 
Additionally, 2D FFT has proven beneficial for fine-tuning large language models \cite{fourierft}. 
Circulant-structured matrix-vector products can also be efficiently computed using 1D FFT and IFFT, facilitating neural network training \cite{circulant}. 
Beyond these applications, FFT also plays a key role in spectral convolution \citep{Spectral} and in approximating global attention mechanisms \citep{Fnet}.

From the perspective of automatic differentiation, each operator can be seen as a neural network layer, with or without trainable parameters. 
For example, the low-rank fine-tuning method for large language models \cite{lora} represents a full weight matrix using two low-rank matrices, which is equivalent to introducing two linear layers, thereby requiring storage for intermediate activations. 
Similarly, circulant-structured weight matrices \cite{circulant} leverage FFT and IFFT to transform inputs and parameters. 
Although FFT operators themselves does not contain trainable parameters, they still require the preservation of intermediate results to support automatic differentiation and to handle complex-valued data types that is different from the real-valued inputs. 
Moreover, libraries such as FFTW \cite{fftw} and cuFFT \cite{cufft_layout} demand pre-allocated memory buffers of size $N+2$ real numbers. 
This memory pre-allocation should be handled during the neural model loading phase, which complicates integration and limits the practicality of using these libraries in real-world neural network applications.
Besides, they also does not support bfloat16 data type which is common in modern neural networks.

\section{Preliminary}

\subsection{Standard FFT and rFFT}
Fast Fourier Transform (FFT) is a computationally efficient algorithm for computing the Discrete Fourier Transform and its inverse. Given a sequence of $N$ real or complex values $x(n), n=0,1,\dots, N-1$, the FFT and inverse FFT are defined as follows:

\begin{equation}
\label{eq:Standard FFT}
y_k
=
\sum_{n=0}^{N-1} x_n \cdot e^{-i \frac{2\pi}{N} kn}, 
\quad
x_n 
=
\frac{1}{N} \sum_{k=0}^{N-1} y_k  \cdot e^{i \frac{2\pi}{N} kn}.
\end{equation}

This formulation, referred to as the standard FFT, transforms $N$ input elements into $N$ complex outputs. 
However, when the input is real, the FFT exhibits Hermitian symmetry, i.e., $y_{N-k} = \overline{y_k}$. 
Real-valued input based FFT (rFFT) exploits this property by computing only the first $ N/2 + 1$ complex values.

The rFFT significantly reduces the memory requirement compared to the complex-valued FFT, decreasing the output size from $2N$ real numbers (i.e., $N$ complex numbers) to $N + 2$ real numbers. 
This improvement is based on the following fundamental property of the FFT:

\begin{theorem}[Conjugate Symmetry of Real FFT \citep{conj_rDFT}]
Let $x \in \mathbb{R}^N$ be a real-valued sequence, and let $y_k$ denote its Fast Fourier Transform (FFT). Then the FFT output satisfies the conjugate symmetry property:
\begin{equation}
y_{N-k} = \overline{y_k}, \quad \text{for all } k = 1, 2, \dots, \left\lfloor \frac{N}{2} \right\rfloor.
\end{equation}
\end{theorem}

This property implies that the FFT of a real-valued signal is \emph{redundant}—the full frequency-domain spectrum can be uniquely reconstructed from only the first $\left\lfloor \frac{N}{2} \right\rfloor + 1$ complex coefficients:
$y_0, y_1, \dots, y_{\left\lfloor \frac{N}{2} \right\rfloor}.$
Leveraging this redundancy, rFFT implementations store only the non-redundant half of the spectrum, thereby reducing both computation and memory footprint. 

However, this comes at the cost of mismatched memory sizes between the input and output: in the FFT, $N$ real-valued inputs are transformed into outputs that take the memory space of $N+2$ real values; in the IFFT, the complex inputs taking the space of $N+2$ real values are mapped back to $N$ real outputs. 
In both directions, the input and output cost different amount of memory. 

This memory misalignment is difficult for neural network training, where tensors are generally allocated with fixed shapes. 
The requirement to expand an $N$-element real tensor to $N+2$ elements necessitates either pre-allocation or runtime reallocation, which causes integration difficulty, prevents true in-place computation and potentially incur significant overhead in memory-constrained environments.

\begin{figure}[tb]
\centering
\includegraphics[width=\textwidth]{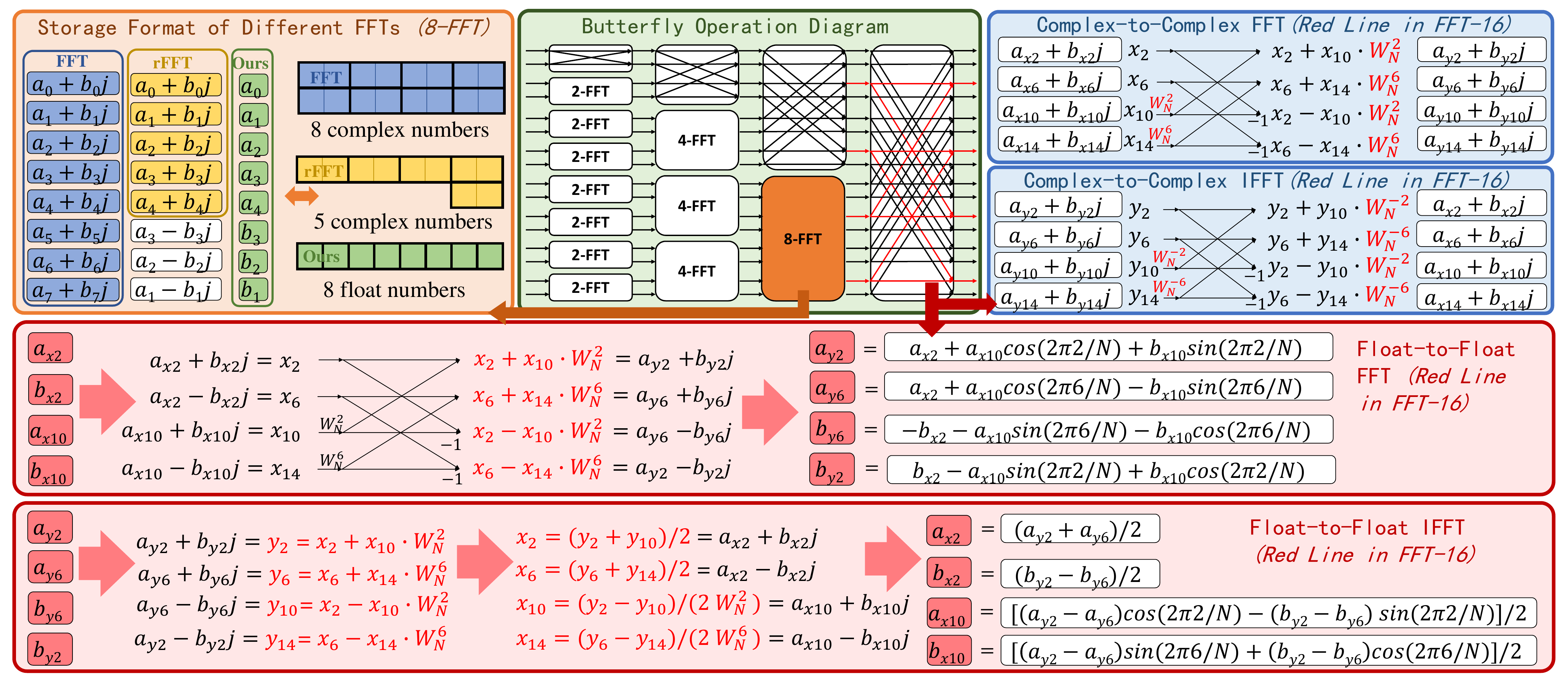}
\caption{Overview of our method and its differences from standard FFT and rFFT implementations. The green section depicts the Butterfly Operation Diagram using a 16-point FFT (16-FFT) as an example. The orange section illustrates the storage formats of different FFT implementations, shown on an 8-point FFT (8-FFT). Two representative butterfly computation paths in the 16-FFT are highlighted in red, and expanded into: (i) the blue section showing Complex-to-Complex FFT and IFFT operations, and (ii) the red section showing Float-to-Float FFT and IFFT operations—both derived from the red paths in the 16-FFT diagram. This figure summarizes the key computational flows and memory layouts addressed by our in-place real-domain FFT design.}
\label{fig:overview}
\end{figure}

\subsection{In-place Transform via Butterfly Operation}
The FFT can be efficiently computed using the Cooley–Tukey algorithm \cite{cooley1965algorithm}, which recursively decomposes given FFT into smaller FFTs. 
As illustrated in the Butterfly Operation Diagram section of Fig.\ref{fig:overview}, a 16-point FFT is first recursively decomposed into smaller FFTs—two 8-point FFTs, then four 4-point FFTs, and finally eight 2-point FFTs. 
Following this hierarchical decomposition, the FFT is computed through a series of butterfly operations applied at each level of the recursion.
At the core of the algorithm lies the butterfly operation, a fundamental computation that transforms a pair of complex values into two outputs using addition, subtraction, and multiplication by a twiddle factor.

As illustrated in the blue-shaded region of Fig.\ref{fig:overview}, the ``Complex-to-Complex FFT'' section provides two butterfly operations, both originating from FFT16, which are highlighted in red. The butterfly operation for an $r$-point FFT is defined as:

\begin{equation}
\label{eq:FFT_r}
\begin{split}
y_k &= x_k + x_{k + \frac{r}{2}} \cdot W_N^m, \\
y_{k + \frac{r}{2}} &= x_k - x_{k + \frac{r}{2}} \cdot W_N^m,
\end{split}
\end{equation}

where $W_N^m$ is the twiddle factor, $N$ is the total size of the FFT, and $r$ is the size of the current sub-FFT. The exponent $m$ is given by $m = 1 + \log_2 \left( N/r \right)$. 
For each $r$-point FFT, there are $\frac{r}{2}$ such butterfly pairs, with $k = 0, 1, \ldots, r/2 - 1$. The same butterfly operations are used in the ``Complex-to-Complex IFFT'' section, differing only in the choice of twiddle factors. 
The overall structure of the Cooley–Tukey FFT and its inverse (IFFT) is almost the same, except that the IFFT applies a final normalization factor of $1/N$.

It can be noticed that butterfly operations are inherently in-place, allowing outputs to overwrite inputs without the need for extra memory. 
This makes FFT especially suitable for memory-efficient implementations. 
However, since most neural network computations involve real-valued tensors, converting these to complex-valued representations usually requires additional memory allocation. 
With real-valued input and in-place constraint, given the aforementioned memory mismatch, the butterfly process is not directly applicable. 

\subsection{FFT based Model Training via Circulant Matrix}
The circulant matrix based neural network training can be converted to FFT and IFFT operations for acceleration \cite{cheng2015exploration}.
Given a real-valued circulant weight matrix $\mathbf{C} \in \mathbb{R}^{N \times N}$ defined by its first column $\mathbf{c} \in \mathbb{R}^N$, the linear transformation $\mathbf{y} = \mathbf{C} \mathbf{x}$ can be equivalently computed in the frequency domain as follows:
\begin{equation}
\label{eq:circ_forward}
\mathbf{y} = \texttt{IFFT}\left( \texttt{FFT}(\mathbf{c}) \odot \texttt{FFT}(\mathbf{x}) \right),
\end{equation}
where $\odot$ denotes elementwise multiplication.

This structure not only accelerates computation but also simplifies the gradient computation during training. 
By leveraging the linearity and conjugate of the FFT results \cite{circulant}, gradients with respect to both input $\mathbf{x}$ and parameter $\mathbf{c}$ can be computed by:
\begin{align}
\label{eq:circ_backward}
\frac{\partial \mathcal{L}}{\partial \mathbf{x}} &= 
\texttt{IFFT} \left( \overline{ \texttt{FFT}(\mathbf{c}) } \odot \texttt{FFT} \left( \frac{\partial \mathcal{L}}{\partial \mathbf{y}} \right) \right), ~~~~~~
\frac{\partial \mathcal{L}}{\partial \mathbf{c}} = 
\texttt{IFFT} \left( \overline{ \texttt{FFT}(\mathbf{x}) } \odot \texttt{FFT} \left( \frac{\partial \mathcal{L}}{\partial \mathbf{y}} \right) \right),
\end{align}
where $\overline{\texttt{FFT}(\cdot)}$ means taking the conjugate of the FFT results. 
There is also block circulant matrix based training that aims to fit non-square matrix \cite{ding2017circnn}, where the matrix is divided into blocks by partition size $p$.

\section{Method}

\subsection{In-place FFT with Real-valued Input}

While the Cooley–Tukey algorithm enables efficient in-place computation of the Fast Fourier Transform (FFT), it operates inherently in the complex domain. However, in modern neural networks, the vast majority of model parameters are real-valued. Transitioning between real and complex representations introduces unnecessary memory overhead, especially in memory-constrained training settings.
To address this issue, we propose new strategies to reduce the memory footprint of Fourier-based transformations at the layer level. 

\paragraph{Conjugate Symmetry in Cooley–Tukey Sub-FFTs.}
The Cooley–Tukey algorithm recursively decomposes a length-$N$ FFT into smaller FFTs. 
If the original input sequence $x \in \mathbb{R}^N$ is real-valued, then each recursively computed sub-FFT also receives a real-valued input (or a linear combination of real values). 
Since the FFT is a linear operation and conjugate symmetry is preserved under linear combinations, each sub-FFT applied to real-valued data also satisfies the conjugate symmetry property:
\begin{equation}
y_{r - k} = \overline{y_{k}}, \quad \text{for all } k = 1, \dots, \left\lfloor \frac{r}{2} \right\rfloor,
\end{equation}
where $y$ denotes the output of a size-$r$ FFT block.
Thus, conjugate symmetry is preserved at every level of the Cooley–Tukey decomposition. This property enables memory-efficient computation in real-input FFTs throughout the recursive stages.

\paragraph{Squeeze $N+2$ into $N$.}
Although the rFFT reduces memory usage of complex-valued FFT by exploiting conjugate symmetry, it stores $N/2 + 1$ complex numbers in the memory space of $N+2$ real numbers, which is different from the the original real-valued input of length $N$. 
This result does not allow for true in-place computation within the original real-valued buffer.
We further analyze the structure of the sub-FFTs in the Cooley–Tukey decomposition. 
For any $r$-point FFT of real-valued input, the output satisfies:
\[
y_0,\ y_{r/2} \in \mathbb{R}, \quad \text{and} \quad y_k = \overline{y_{r - k}},\ \text{for } k = 1, \dots, r/2 - 1.
\]

This implies that only the real parts of $y_0$ and $y_{r/2}$ need to be stored.
For remaining $r-2$ complex values, we only store the real and imaginary parts of $y_1, \dots, y_{r/2 - 1}$ , and the rest can be reconstructed via conjugation.
As a result, we successfully reduce the memory space of $N+2$ real values to the size of $N$ real values. 

\paragraph{Memory Layout Design.}
Based on aforementioned observation, we propose a data layout design where each complex coefficient $y_k$ ($1 \leq k < r/2$) stores its real part at index $k$ and stores its imaginary part at the conjugate-symmetric index $r - k$. In this way, the entire frequency-domain representation fits into a real-valued buffer of size $r$, without requiring any complex-valued memory. 
The special cases $y_0$ and $y_{r/2}$, which are always real, each occupy a single real-valued slot. 
The ``Storage Format of Different FFTs'' in Fig.\ref{fig:overview} illustrates our layout design in comparison to standard FFT and rFFT formats.

\begin{proposition}
\label{prop:conjugate-propagation}
In the Cooley–Tukey FFT algorithm on real-valued input, at each stage of the recursive decomposition, every conjugate-symmetric pair and its butterfly counterparts form a symmetric four-element group. This structural symmetry ensures that the butterfly operations preserve conjugate symmetry and can be performed entirely in-place.
\end{proposition}

\begin{proof}
Consider two conjugate-symmetric outputs $x(a_1)$ and $x(b_1)$ from an $m$-point FFT stage, where their indices are given by
\[
a_1 = 2km + \tfrac{m}{2} - i, \quad b_1 = 2km + \tfrac{m}{2} + i
\]
for some integers $k$ and $i$. In the next $2m$-point FFT stage, these values participate in butterfly operations with their counterparts at
\[
a_2 = (2k+1)m + \tfrac{m}{2} - i, \quad b_2 = (2k+1)m + \tfrac{m}{2} + i.
\]

Define the center of the $2m$-point FFT block as $c = (2k+1)m$. Then, the offsets from the center are:
\begin{align*}
c - a_1 &= \tfrac{m}{2} + i, &
c - b_1 &= \tfrac{m}{2} - i, \\
c - a_2 &= -\tfrac{m}{2} + i, &
c - b_2 &= -\tfrac{m}{2} - i.
\end{align*}

This confirms that the set $\{a_1, b_1, a_2, b_2\}$ is symmetric with respect to the center index $c$, and forms two conjugate-symmetric pairs placed at symmetric offsets.Since the FFT butterfly operations preserve conjugate symmetry when applied to conjugate inputs, the output values at this stage also maintain the same symmetry.

Therefore, all computations involving these four values can be performed in-place, and the symmetric layout remains valid in the next FFT stage.
\end{proof}

This symmetry ensures that at every level of the decomposition, each butterfly involving a conjugate pair produces another conjugate pair. As a result, conjugate symmetry is recursively preserved across all stages of the Cooley–Tukey FFT algorithm.

As a concrete example, consider the case shown in Fig.~\ref{fig:overview} (Float-to-Float FFT), where an 8-point FFT stage has a conjugate pair located at indices $a = 2$ and $b = 6$. In the following 16-point FFT stage, these values participate in butterfly operations with elements at indices $10$ and $14$. The resulting index pairs $(2, 14)$ and $(6, 10)$ are symmetric with respect to the center of the 16-point block and thus also form conjugate pairs.

Consequently, the interleaved memory layout—where the real part of $y_k$ is stored at index $k$ and the imaginary part at index $r - k$—remains consistent throughout the entire FFT computation. This enables the algorithm to be executed entirely in-place within a real-valued buffer, eliminating the need to explicitly store redundant conjugate components.

Based on the recursive structure and symmetry, we summarize our in-place real-domain FFT algorithm:
\begin{enumerate}
\item Store real and imaginary parts of conjugate pairs in an interleaved layout within the input real-valued memory space;
\item Perform all butterfly operations entirely in-place without auxiliary buffers;
\item Reconstructs the complex frequency spectrum from a real input of length $N$.
\end{enumerate}
Overall, our design enables fully in-place FFT computation for real-valued input, with zero memory overhead and compatible with the Cooley–Tukey algorithm.

\subsection{In-place IFFT with Symmetric Complex-valued Input}  
While the FFT benefits from conjugate symmetry in real-valued inputs, the inverse transform receives conjugate-symmetric complex values as input. 
However, unlike FFT computation, the outputs of sub-IFFTs in the Cooley–Tukey recursion are not guaranteed to be real-valued or conjugate-structured. 
This makes it difficult to directly apply the aforementioned in-place memory sharing mechanism.

To overcome the input difference in IFFT, we exploit the linearity of the FFT. 
Since each butterfly operation is a linear combination of its inputs, we can reverse the forward FFT computation structure to recover the original real-valued signal. 
Specifically, we compute the inverse using the same butterfly graph but reverse the direction of data flow and scale the results appropriately.

As illustrated in Fig.\ref{fig:overview} (Float-to-Float IFFT section), we compute intermediate outputs from the complex input $Y$ as follows:
\begin{equation}
\begin{split}
x_2 &= \frac{y_2 + y_{10}}{2}, \quad x_6 = \frac{y_6 + y_{14}}{2}, \\
x_{10} &= \frac{y_2 - y_{10}}{2 W_N^2}, \quad x_{14} = \frac{y_6 - y_{14}}{2 W_N^6},
\end{split}
\end{equation}
where $W_N^k = \exp\left(-2\pi i k / N\right)$ denotes the $k$-th twiddle factor.

In this formulation, we split each conjugate pair into symmetric and anti-symmetric components, enabling recovery of the original signal through only real-domain operations. By carefully reusing buffer locations during this process, we implement the in-place IFFT computation without auxiliary memory for intermediate complex arrays.

\paragraph{Symmetry in Circulant Matrix based Training.}
According to Eq. \ref{eq:circ_forward}, even though the FFT results are naturally symmetric for real-valued inputs, there is elementwise multiplication between two FFT results. 
Given that $\overline{A\cdot B}=\overline{A}\cdot\overline{B}$, it follows that the elementwise multiplication result maintains the symmetry property as in FFT results of input and circulant weight vector. 
Therefore, the IFFT operation input in Eq. \ref{eq:circ_forward} and Eq. \ref{eq:circ_backward} are with symmetric complex-valued input.

\section{Experiments}
All our experiments compare three different FFT implementations:
\textbf{(1) fft:} standard complex-valued FFT using \texttt{torch.fft.fft}/\texttt{ifft} from PyTorch \citep{paszke2017automatic};
\textbf{(2) rfft:} real-input FFT using \texttt{torch.fft.rfft}/\texttt{irfft}, exploiting Hermitian symmetry;
\textbf{(3) ours:} custom CUDA-based real-domain FFT with in-place forward/backward implementation, reusing the input real-valued memory for intermediate result storage.

We also include two common baselines:
\textbf{(1) FF}: updating all trainable parameters;
\textbf{(2) lora}\citep{lora}: low-rank adaptation with parameter-efficient updates.
For fair comparison, all runs use the same training configuration (batch size, optimizer, precision). 

\subsection{Memory Efficiency} 
To evaluate the memory efficiency of our proposed in-place training, we conduct experiments in two settings:
\textbf{(1) single-layer analysis:} 
we perform training on a singular linear layer with different training methods, where circulant matrix based training \citep{circulant} are accomplished with different FFT implementations;
\textbf{(2) full-model training:} we apply the circulant fine-tuning approach \citep{circulant} to RoBERTa-large and LLaMA2-7B and monitor memory usage throughout training.
All circulant variants share the same number of trainable parameters, differing only in FFT backend. 

Peak memory is recorded using PyTorch memory profiler.
To better understand the memory distribution, we also visualize the breakdown of memory usage (model weights, trainable params, gradients, others) in both the single-layer and full-model training.

\subsubsection{Single Layer Training}

\begin{table}[tb]
\caption{Peak GPU memory usage (in MB) measured during single layer training (up to the end of the backward pass) for different methods under varying input shapes. 
For inputs of shape $D=4096$, LoRA uses rank 64; for $D=1024$, the rank is 32. 
Entries marked as “N/A” indicate that the specified block size (e.g., 4096) is incompatible with the given input shape (e.g., 1024).
Values in parentheses denote how many times memory is reduced compared to full fine-tuning.}
\label{tab:peak_memory_single}
\centering
\setlength{\tabcolsep}{3pt} 
\resizebox{\textwidth}{!}{
\begin{tabular}{l|c c c |c c c}
\toprule
GPU Mem. & \multicolumn{3}{c|}{D = 4096} & \multicolumn{3}{c}{D = 1024} \\
 (MB) & {B=1} & {B=16} & {B=256} & {B=1} & {B=16} & {B=256} \\
\midrule
full-finetune & $144.33$ & $145.50$ & $164.25$ & $24.27$ & $24.56$ & $29.25$ \\
\midrule
$\text{lora}$ & $20.31(\times 7.11)$ & $21.25(\times 6.85)$ & $39.38(\times 4.17)$ & $16.77(\times 1.45)$ & $17.00(\times 1.44)$ & $21.69(\times 1.35)$ \\
\midrule
$\text{fft}_{p=128}$ & $3.65(\times 39.55)$ & $35.88(\times 4.06)$ & $551.50(\times 0.30)$ & $0.25(\times 95.22)$ & $2.66(\times 9.22)$ & $41.22(\times 0.71)$ \\
$\text{rfft}_{p=128}$ & $3.14(\times 45.93)$ & $35.14(\times 4.14)$ & $547.13(\times 0.30)$ & $0.22(\times 111.20)$ & $2.53(\times 9.72)$ & $40.30(\times 0.73)$ \\
$\text{ours}_{p=128}$ & \textbf{1.06($\times 135.78$)} & \textbf{2.00($\times 72.73$)} & \textbf{20.50($\times 8.01$)} & \textbf{0.08($\times 308.73$)} & \textbf{0.34($\times 71.35$)} & \textbf{5.03($\times 5.81$)}  \\
\midrule
$\text{fft}_{p=256}$ & $1.89(\times 76.24)$ & $19.03(\times 7.65)$ & $293.25(\times 0.56)$ & $0.15(\times 166.80)$ & $1.61(\times 15.24)$ & $25.08(\times 1.17)$  \\
$\text{rfft}_{p=256}$ & $1.62(\times 89.17)$ & $18.35(\times 7.93)$ & $286.06(\times 0.57)$ & $0.12(\times 194.92)$ & $1.48(\times 16.63)$ & $24.02(\times 1.22)$ \\
$\text{ours}_{p=256}$ & \textbf{0.56($\times 256.36$)} & \textbf{1.50($\times 96.97$)} & \textbf{20.25($\times 8.11$)} & \textbf{0.05($\times 512.42$)} &\textbf{0.33($\times 74.74$)} & \textbf{5.02 ($\times 5.83$)} \\
\midrule
$\text{fft}_{p=512}$ & $1.02(\times 141.97)$ & $10.63 (\times 13.68)$ & $164.50(\times 1.00)$ & $0.09 (\times 267.23)$ & $1.09 (\times 22.60)$ & $17.03(\times 1.72)$ \\
$\text{rfft}_{p=512}$ & $0.86(\times 167.28)$ & $10.03 (\times 14.51)$ & $156.66(\times 1.05)$ & $0.08 (\times 312.61)$ & $0.96 (\times 25.68)$ & $15.05(\times 1.94)$ \\
$\text{ours}_{p=512}$ & \textbf{0.31($\times 461.14$)} & \textbf{1.38($\times 105.78$)} & \textbf{20.13 ($\times 8.16$)} & \textbf{0.03 ($\times 764.69$)} & \textbf{0.32 ($\times 76.56$)} & \textbf{5.01($\times 5.84$)}  \\
\midrule
$\text{fft}_{p=1024}$ & $0.58(\times 249.23)$ & $6.44(\times 22.59)$ & $100.22(\times 1.64)$ & $0.06 (\times 382.35)$ & $0.83 (\times 29.76)$ & $13.01(\times 2.25)$ \\
$\text{rfft}_{p=1024}$ & $0.49(\times 295.88)$ & $5.88(\times 24.73)$ & $92.24 (\times 1.78)$ & $0.05 (\times 447.79)$ & $0.70 (\times 35.15)$ & $11.02(\times 2.65)$ \\
$\text{ours}_{p=1024}$ & \textbf{0.19($\times 767.76$)} & \textbf{1.31($\times 110.82$)}& \textbf{20.06 ($\times 8.19$)} & \textbf{0.02 ($\times 1,014.39$)} & \textbf{0.32 ($\times 77.50$)} & \textbf{5.00($\times 5.84$)}  \\
\midrule
$\text{fft}_{p=4096}$ & $0.25(\times 575.07)$ & $3.30(\times 44.12)$ & $52.05(\times 3.16)$ & N/A & N/A & N/A \\
$\text{rfft}_{p=4096}$ & $0.21(\times 698.79)$ & $2.78(\times 52.25)$ & $44.04(\times 3.73)$ & N/A & N/A & N/A  \\
$\text{ours}_{p=4096}$ & \textbf{0.09($\times 1,531.54$)} & \textbf{1.27($\times 114.92$)} & \textbf{20.02($\times 8.21$)} & N/A & N/A & N/A  \\
\bottomrule
\end{tabular}
}
\end{table}

\textbf{Setups.}
To isolate the memory overhead introduced by different FFT implementations, we conduct controlled experiments on a single fine-tuned layer using an NVIDIA A100 GPU. 
We vary the input dimension $\mathbf{x} \in \mathbb{R}^{B \times D}$, with $D \in \{1024, 4096\}$ and batch size $B \in \{1, 16, 256\}$. 
For circulant-based methods, we further vary the block size to evaluate its influence on memory usage. 
Peak memory is recorded during both forward and backward passes, and the results are reported in Tab.\ref{tab:peak_memory_single}.
Fig.\ref{fig:singlelayer_breakdown} shows the memory breakdown for the large setup ($D=4096$) under batch sizes $B=1$ and $B=256$, highlighting the memory footprint of intermediate tensors created during forward computation and gradients allocated during backpropagation. 
This breakdown reveals the impact of in-place FFT on reducing transient memory usage.

\begin{figure}[!t]

\centering
\subfloat[Memory allocation ($B = 1$)]{ 
\includegraphics[width=0.48\textwidth]{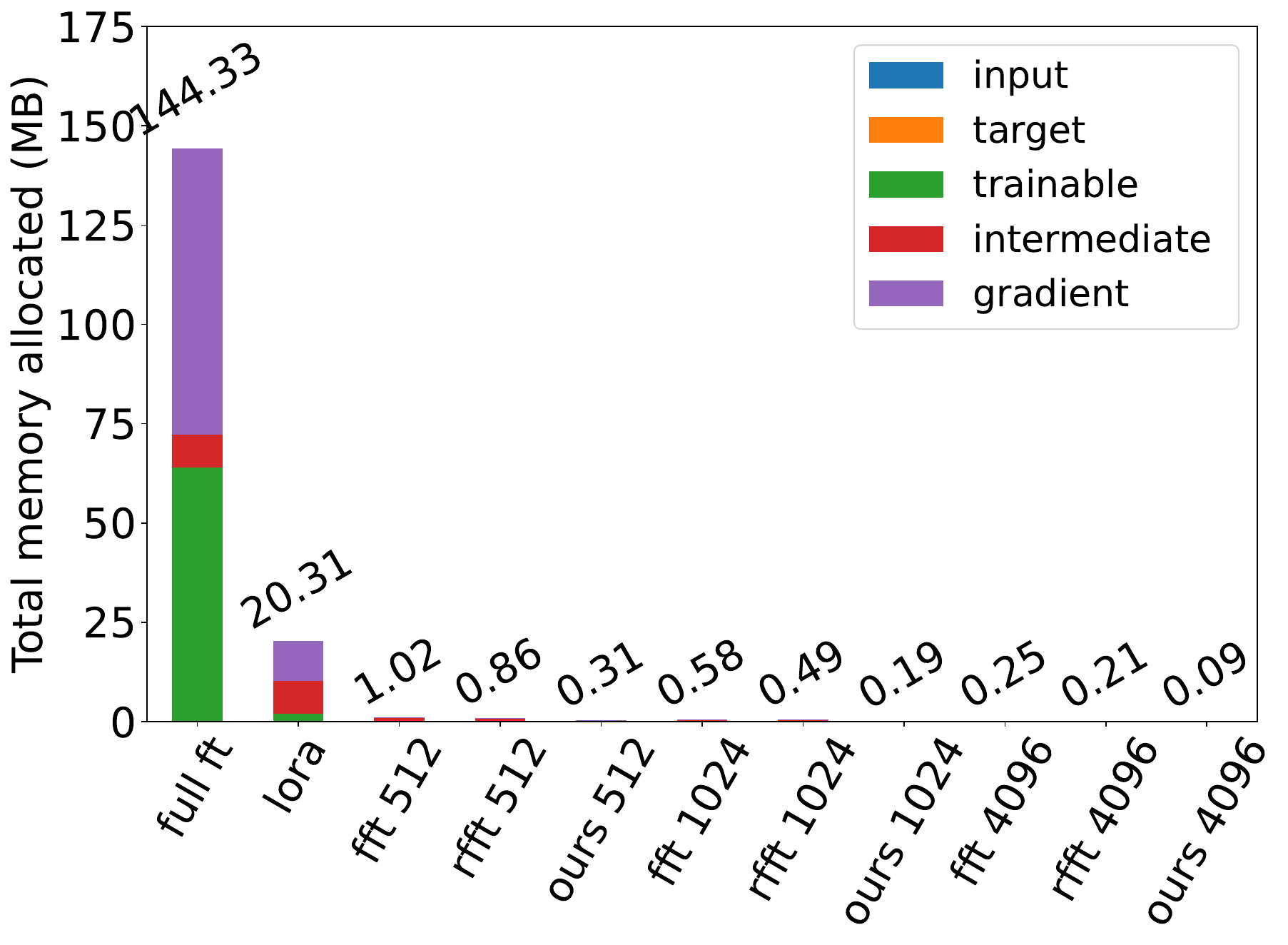}
}
\subfloat[Memory allocation ($B = 256$)]{
\includegraphics[width=0.48\textwidth]{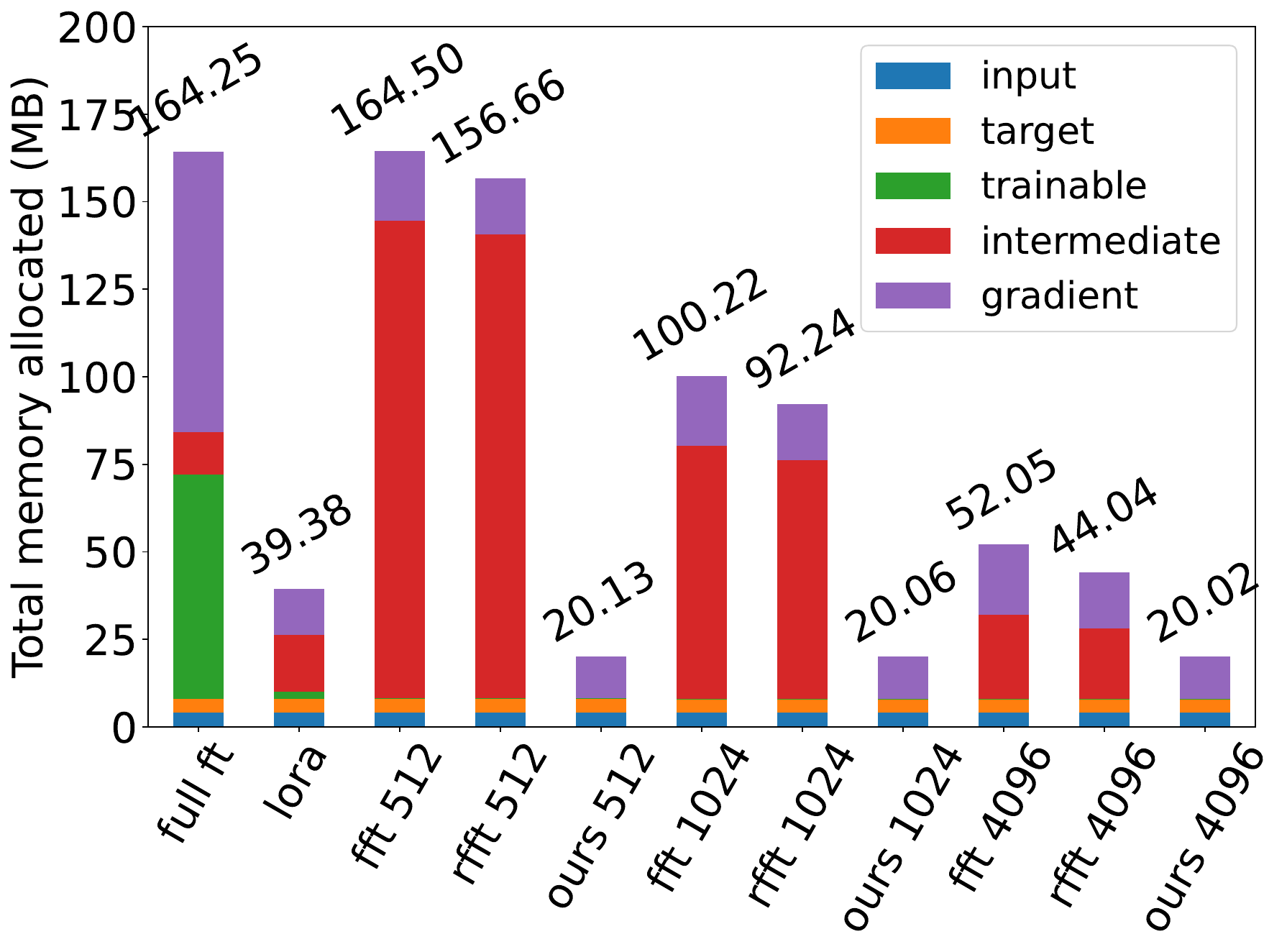}
}

\caption{
Memory breakdown during single-layer fine-tuning with hidden dimension $D=4096$, under two batch sizes: (a) $B = 1$ and (b) $B = 256$. \textit{Intermediate tensors} are allocated during the forward pass, while \textit{gradients} appear in the backward pass. This illustrates how batch size impacts memory allocation for activations and gradients.
}
\label{fig:singlelayer_breakdown}
\end{figure}

\textbf{Results.}
Tab.\ref{tab:peak_memory_single} and Fig.\ref{fig:singlelayer_breakdown} clearly demonstrate the effectiveness of our in-place method. 
When the batch size is small, all circulant variants show memory advantages over full fine-tuning and LoRA. However, as the block size decreases and the batch size increases, standard FFT-based circulant layers incur increasing overhead from intermediate tensors, especially during the forward pass.
In contrast, our method performs the entire forward computation in-place, introducing no intermediate tensors. As shown in Fig.\ref{fig:singlelayer_breakdown}, this leads to significant memory savings in the single-layer setup. Moreover, by overwriting the \texttt{grad\_output} in-place at the final stage of the backward pass, our method also reduces memory usage during gradient computation.


\subsubsection{Full Model Training}

\begin{table}[tb]
\caption{Peak GPU memory usage across different training stages during one epoch on LLaMA2-7B and RoBERTa-large. 
\textit{model} indicates memory used to load the base model; \textit{trainable} refers to memory allocated for trainable parameters; 
\textit{gradient} denotes the analytically estimated memory for gradients of trainable parameters; 
\textit{others} represents the remaining memory, computed as the difference between the peak usage and the sum of the above three, accounting for buffers, activations, and miscellaneous overhead.}
\label{tab:model_mem}
\centering
\setlength{\tabcolsep}{3pt} 
\resizebox{\textwidth}{!}{
\begin{tabular}{l|c c c c c|l|c c c c c}
\toprule
\multirow{3}{*}{Method} & \multicolumn{5}{c|}{LLaMA2-7b } & \multirow{3}{*}{Method} & \multicolumn{5}{c}{RoBERTa-Large} \\
~ & model & trainable & gradient & others & total & ~ & model & trainable & gradient & others & total\\
~ & (GB) & (MB) & (MB) & (GB) & (GB) & ~ & (GB) & (MB) & (MB) & (GB) & (GB) \\
\midrule
FF & $12.61$ & $0.00$ & $6144.00$ & $8.28$ & $26.90$ & FF & $ 1.33$ & $0.00$ & $ 192.00$ & $4.63$ & $6.15$ \\
\midrule
$\text{lora}_{r=32}$ & $12.61$ & $48.00$ & $ 96.00$ & $6.20$ & $ 18.96$ & $\text{lora}_{r=8}$ & $1.33$ & $3.00$ & $ 3.00$ & $4.90$ & $ 6.24$ \\
$\text{lora}_{r=64}$ & $12.61$ & $96.00$ & $ 192.00$ & $ 6.26$ & $19.15$ & $\text{lora}_{r=16}$ & $1.33$ & $6.00$ & $ 6.00$ & $4.92$ & $6.26$ \\
\midrule
$\text{fft}_{p=512}$ & $12.61$ & $6.00$ & $12.00$ & $8.17$ & $20.81$ & $\text{fft}_{p=256}$ & $1.33$ & $0.75$ & $0.75  $ & $5.39$ & $6.72$ \\
$\text{rfft}_{p=512}$ & $12.61$ & $6.00$ & $12.00$ & $6.65$ & $19.28$ & $\text{rfft}_{p=256}$ & $1.33$ & $0.75$ & $0.75$   & $4.80$ & $ 6.13$ \\
$\text{ours}_{p=512}$ & \textbf{12.61} & \textbf{6.00} & \textbf{12.00} & \textbf{5.30} & \textbf{ 17.93} & $\text{ours}_{p=256}$ & \textbf{1.33} & \textbf{0.75} & \textbf{0.75}  & \textbf{4.44} & \textbf{5.77} \\
\midrule
$\text{fft}_{p=1024}$ & $12.61$ & $3.00$ & $6.00$ & $6.58$ & $19.20$ & $\text{fft}_{p=512}$ & $ 1.33$ & $0.38$ & $0.38$   & $5.41$ & $6.74$ \\
$\text{rfft}_{p=1024}$ & $12.61$ & $3.00$ & $6.00$ & $6.55$ & $19.17$ & $\text{rfft}_{p=512}$ & $1.33$ & $0.38$ & $0.38$   & $4.76$ & $6.08$ \\
$\text{ours}_{p=1024}$ & \textbf{12.61} & \textbf{3.00} & \textbf{6.00} & \textbf{5.30} & \textbf{ 17.92} & $\text{ours}_{p=512}$ & \textbf{1.33} & \textbf{0.38} & \textbf{0.38}   & \textbf{4.44} & \textbf{5.77} \\
\midrule
$\text{fft}_{p=4096}$ & $12.61$ & $0.75$ & $1.50$ & $8.09$ & $20.71$ & $\text{fft}_{p=1024}$ & $1.33$ & $0.19$ & $0.19$   & $5.39$ & $6.72$ \\
$\text{rfft}_{p=4096}$ & $12.61$ & $0.75$ & $1.50$ & $6.55$ & $19.16$ & $\text{rfft}_{p=1024}$ & $1.33$ & $0.19$ & $0.19$   & $4.76$ & $6.09$ \\
$\text{ours}_{p=4096}$ & \textbf{ 12.61} & \textbf{0.75} & \textbf{1.50} & \textbf{5.29} & \textbf{17.91} & $\text{ours}_{p=1024}$ & \textbf{ 1.33} & \textbf{0.19} & \textbf{0.19}   & \textbf{4.44} & \textbf{5.77} \\
\bottomrule
\end{tabular}
}
\end{table}

\textbf{Setups.}
We conduct full-model experiments on both LLaMA2-7B and RoBERTa-large using an NVIDIA A100 GPU. 
For LLaMA2-7B, we use the GSM8K dataset with \texttt{per\_device\_train\_batch\_size} set to 2 and \texttt{gradient\_accumulation\_steps} set to 4. 
For RoBERTa-large, we use the MRPC dataset with a batch size of 32. 
These configurations follow the standard precision training setups for each task \cite{circulant}, and we retain them to reflect real-world GPU memory cost. 
We use stochastic gradient descent (SGD) as the optimizer in all experiments.

\textbf{Results.}
From Tab.\ref{tab:model_mem}, we observe that all methods have the same \textit{base model} memory, and the memory for \textit{trainable\_params} is negligible compared to the total model size. 
It can be noticed that \textit{others} take up the second largest memory consumption, which is managed by PyTorch framework for storing activations, dynamic memory allocation and release, etc.  
In LLaMA2-7B, \textit{gradient} memory is approximately twice that of \textit{trainable\_params} because the forward pass uses \texttt{bf16} to reduce memory usage, but gradients must be stored in \texttt{float32} as backward computations do not support \texttt{bf16}. 
In contrast, RoBERTa-large uses full-precision training, so the gradient memory matches the parameter memory size.

It is also worth noting that fft and rfft implementations do not support \texttt{bf16} arithmetic, limiting their memory optimization potential during the forward pass. While full fine-tuning does not incur extra memory for \textit{trainable\_params} (as the base model is updated directly), it still requires large \textit{gradient} storage due to the number of trained parameters.

Our method, in contrast, outperforms all FFT and inverse FFT operations in-place with native support for real-valued \texttt{bf16} input. 
Besides, our method also outperforms LoRA adapters which has been widely adopted due to its small parameter amount and memory consumption.
As a result, it consistently achieves the lowest peak memory usage across training steps, showing superior memory efficiency in practical fine-tuning setups.

\subsection{Runtime Efficiency and Numerical Accuracy}

To comprehensively evaluate the proposed \textbf{rdfft} operator, we analyze both its low-level computational efficiency and numerical accuracy (operator-level), as well as its end-to-end performance when integrated into large-scale fine-tuning tasks (model-level). This dual perspective allows us to capture both micro-level operator behavior and macro-level training characteristics.

\subsubsection{Operator-Level Evaluation}

At the operator level, Tab.\ref{tab:operator-level} presents both runtime and numerical accuracy.  
At small to medium sizes (512 and 1024), ours achieves competitive runtime with rfft, suggesting that in-place execution can be efficient when synchronization cost is limited. At larger sizes (4096), however, the runtime overhead increases due to CUDA thread-block limitations, where synchronization is required both within and across blocks. Moreover, the inverse transform (ours) is faster than the forward one, since it reuses the butterfly structure in reverse order, thereby reducing dependencies.

In terms of accuracy, since \text{rdfft} only reformulates the memory layout to support true in-place execution, it introduces no information loss and preserves the mathematical equivalence of the Fourier transform. 
Tab.\ref{tab:operator-level} shows that both absolute and relative errors remain at the level of floating-point numerical noise, confirming that \text{rdfft} faithfully reproduces the FFT spectrum.

\begin{table}[tb]
\centering
\caption{Standalone operator runtime and numerical accuracy of FFT variants. 
Runtime (RT, in ms) is measured on an A800 GPU with FP32 precision, averaged over 1000 runs. 
Each operator is evaluated for both forward and inverse transforms, as shown in the table. 
Operator-level numerical accuracy of \text{rfft} and \text{ours} is evaluated against the \texttt{torch.fft.fft} baseline, with errors reported as absolute and relative values.
Entries marked as “N/A” indicate that the baseline \text{fft} is used for reference, and thus its self-comparison is unnecessary.
}
\label{tab:operator-level}
\resizebox{\textwidth}{!}{
\begin{tabular}{l|l|ccc|ccc|ccc}
\toprule
\multicolumn{2}{c|}{~} & \multicolumn{3}{c|}{p=512} & \multicolumn{3}{c|}{p=1024} & \multicolumn{3}{c}{p=4096} \\
\multicolumn{2}{c|}{Method} & $\text{fft}$ & $\text{rfft}$ & $\text{ours}$ & $\text{fft}$ & $\text{rfft}$ & $\text{ours}$ & $\text{fft}$ & $\text{rfft}$ & $\text{ours}$ \\
\midrule
\multirow{2}{*}{RT} & forward & 0.0246 & 0.0195 & \textbf{0.0279} & 0.0249 & 0.0197 & \textbf{0.0319} & 0.0252 & 0.0199 & \textbf{0.0687}\\
~ & inverse & 0.0325 & 0.0450 & \textbf{0.0233} & 0.0322 & 0.0421 & \textbf{0.0272} & 0.0327 & 0.0470 & \textbf{0.0503} \\
\midrule
\multirow{2}{*}{Acc.} & absolute & N/A & 1.88e-07 & \textbf{5.99e-07} & N/A & 1.92e-07 & \textbf{5.75e-07} & N/A & 2.55e-07 & \textbf{5.84e-07}\\
~ & relative & N/A & 0.0001 & \textbf{0.0008} & N/A & 0.0001 & \textbf{0.0005} & N/A & 0.0012 & \textbf{0.0018}\\
\bottomrule
\end{tabular}
}
\end{table}

\subsubsection{Model-Level Evaluation}


\begin{table}[tb]
\centering
\caption{
Runtime throughput and MRPC accuracy of different fine-tuning methods.
Token-level throughput (Thr., in k~tokens/sec) is measured on \text{LLaMA-2-7B} using the GSM8K dataset with one A800 GPU, 
while MRPC classification accuracy (Acc., \%) is evaluated on \text{RoBERTa-large}. 
Accuracy results are reported from the Block-Circulant Adapter (BCA) work \citep{circulant}, where only limited configurations were provided, leading to some “N/A” entries. 
All \text{lora} experiments use rank $r=32$. 
For circulant-based methods, p denotes the block size.
}
\label{tab:model-level}
\begin{tabular}{l|c|c|ccc|ccc|ccc}
\toprule
\multirow{2}{*}{Method} & \multirow{2}{*}{\text{FF}}  & \multirow{2}{*}{\text{lora}} & \multicolumn{3}{c|}{p=512} & \multicolumn{3}{c|}{p=1024} & \multicolumn{3}{c}{p=4096} \\
~ & ~ & ~ & $\text{fft}$ & $\text{rfft}$ & $\text{ours}$ & $\text{fft}$ & $\text{rfft}$ & $\text{ours}$ & $\text{fft}$ & $\text{rfft}$ & $\text{ours}$ \\
\midrule
\text{Thr. (k)} & 3.29 & 3.36 & 1.45 & 1.77 & \textbf{0.92} & 1.45 & 1.77 & \textbf{0.93} & 1.45 & 1.77 & \textbf{0.93} \\
\text{Acc. (\%)} & 90.9 & 90.2 & N/A & 90.7 & \textbf{90.0} & N/A & 89.7 & \textbf{90.3} & N/A & N/A & N/A \\
\bottomrule
\end{tabular}
\end{table}

At the model level, Tab.~\ref{tab:model-level} reports both training throughput and downstream task accuracy.  
While our method exhibit lower throughput compared to \text{fft} and \text{rfft}, they eliminate all intermediate buffer allocations during both forward and backward passes, resulting in substantial GPU memory savings—a key advantage for large-scale fine-tuning under limited hardware.

In downstream evaluation on the MRPC benchmark, all experiments were repeated with multiple random seeds for consistency. Our method achieves task-level accuracy on par with \texttt{rfft} and full fine-tuning, indicating no degradation in learning quality.  
Together with the operator-level evidence, these results verify that \textbf{rdfft} is a reliable drop-in replacement for real-input FFT computations, offering strong memory efficiency while maintaining runtime competitiveness and numerical fidelity.
\section{Conclusion}
In this work, we present rdFFT, a novel real-domain, fully in-place Fourier transform framework designed for memory-efficient neural computation. 
Our method enables seamless integration into modern deep learning pipelines and supports consistent forward and backward passes entirely in the real domain.
Extensive experiments on NLU benchmarks demonstrate that rdFFT significantly reduces memory consumption. 
Our results highlight the potential of operator-level memory optimization as a complementary and lossless strategy to existing model compression methods.
In future work, we plan to extend rdFFT to support broader classes of structured transformations and explore its integration with hardware-aware training frameworks for edge deployment.

\paragraph{Limitations.}
\label{sec_limit}
While our method enables real-valued inputs to undergo Fourier transformation and remain in real-valued storage with a corresponding inverse transform, it inherently encodes frequency-domain information in an implicit form. 
This design is well-suited for use cases that do not require direct manipulation of the complex spectrum. 
However, for scenarios where explicit access to the complex-valued frequency representation is needed—such as spectral filtering or custom frequency-domain operations—additional logic is required to decode the real-valued encoding into a usable complex form. 
Once this conversion is performed, it typically involves casting the data to a complex type, which breaks the in-place memory symmetry and incurs additional memory overhead. 
As such, our framework is most effective in applications where complex-domain access is not strictly required.

\newpage

\section*{Acknowledgments}
This work was ﬁnancially supported by the National Key R\&D Program of China (Grant No. 2024YFA1211400) and the Key Project of Shenzhen Basic Research Program (Grant No. JCYJ20241206180301003).

\bibliographystyle{IEEEtran}
\bibliography{arxiv}

\newpage
\appendix

\end{document}